%% file: ms.tex
\documentclass[conference,letterpaper]{IEEEtran}
\usepackage[bottom,flushmargin]{footmisc}
\include{preamble}

\begin{document}
\title{Sparsity-based Defense against \\ Adversarial Attacks on Linear Classifiers} 

\author{
  \IEEEauthorblockN{Zhinus Marzi\IEEEauthorrefmark{1}\thanks{\IEEEauthorrefmark{1}Joint first authors.}, Soorya Gopalakrishnan\IEEEauthorrefmark{1}, Upamanyu Madhow, Ramtin Pedarsani}

  \IEEEauthorblockA{University of California, Santa Barbara\\
                    Email: \{zhinus\_marzi, soorya, madhow, ramtin\}@ucsb.edu}
}

\maketitle

\input{abstract}

\section{Introduction}
\input{introduction}

\section{Related Work}
\input{related_work}

\section{Problem Formulation}
\input{problem}

\section{Sparsity-based defense}
\input{sparse_theory}

\section{Analytical results}\label{sec:analytical}
\input{analytical_results}

\section{Experimental results}
\input{experiments}

\section*{Acknowledgment}
\input{acknowledgement}

\vspace{-5pt}
\bibliography{ms.bbl}

\end{document}

%% file: preamble.tex
\makeatletter
\renewcommand\subsubsection{\@startsection{subsubsection}{2}{\z@}{1.5ex plus 1.5ex minus 0.5ex}%
{0.7ex plus .5ex minus 0ex}{\normalfont\normalsize\itshape}}%
\makeatother

\usepackage[utf8]{inputenc}
\usepackage{graphicx} 
\usepackage[T1]{fontenc}
\usepackage{url}
\usepackage{ifthen}
\usepackage{bbm}
\usepackage[noadjust]{cite}
\usepackage{amssymb}
\usepackage{algorithm}
\usepackage{algorithmicx}
\usepackage[cmex10]{amsmath}
\usepackage{amsthm}
\usepackage{mathtools}
\usepackage{booktabs}
\usepackage{bm}
\usepackage{enumitem,hyperref}

\DeclareMathAlphabet{\mathcal}{OMS}{cmsy}{m}{n}
\SetMathAlphabet{\mathcal}{bold}{OMS}{cmsy}{b}{n}
\DeclareMathOperator*{\supp}{supp}
\DeclareMathOperator*{\sign}{sgn}
\DeclareMathOperator*{\var}{var}
\DeclareMathOperator*{\cov}{cov}
\DeclareMathOperator*{\E}{E}
\DeclareMathOperator*{\prob}{Pr}
\DeclareMathOperator*{\proj}{\mathcal{P}_K}
\DeclareMathOperator*{\support}{\mathcal{S}_K}
\DeclareMathOperator*{\sparse}{\mathcal{H}_K}
\DeclareMathOperator{\bigo}{{\mathcal{O}}}
\DeclareMathOperator{\smallo}{{\scriptstyle\mathcal{O}}}
\DeclareMathOperator*{\polylog}{polylog}
\let\originalleft\left
\let\originalright\right

\renewcommand{\left}{\mathopen{}\mathclose\bgroup\originalleft}
\renewcommand{\right}{\aftergroup\egroup\originalright}
\newcommand{\bw}{{\bm w}}
\newcommand{\bx}{{\bm x}}
\newcommand{\be}{{\bm e}}

\newcommand{\bpsi}{{\bm \psi}}
\newcommand{\bhatx}{\bm{\hat{x}}}
\newcommand{\bbarx}{\bm{\bar{x}}}
\newcommand{\bbare}{\bm{\bar{e}}}

\newcommand{\subzero}{{\mathchoice{}{}{\scriptscriptstyle}{}0}}

\theoremstyle{plain}
\newtheorem{theorem}{Theorem}
\newtheorem{lemma}{Lemma}
\newtheorem{proposition}{Proposition}

\theoremstyle{definition}

\theoremstyle{remark}
\newtheorem*{remark}{Remark}
\newtheorem*{remarks}{Remarks}

\newtheorem*{assumption}{Assumption}

\usepackage[usenames, dvipsnames]{color}
\usepackage{tikz}
\usetikzlibrary{fit, shapes.geometric, shapes.misc, arrows.meta}
\tikzset{>={Latex[width=1.2mm,length=1.5mm]}}
\usepackage{smartdiagram}
\usesmartdiagramlibrary{additions}
\IEEEoverridecommandlockouts
\makeatletter
\renewcommand\footnoterule{%
 \kern-5pt
 \hrule\@width.5\columnwidth
 \kern4.6pt}
\makeatother

%% file: abstract.tex

\begin{abstract}
 Deep neural networks represent the state of the art in machine learning in a growing number of fields, including vision, speech and natural language processing. However, recent work raises important questions about the robustness of such architectures, by showing that it is possible to induce classification errors through tiny, almost imperceptible, perturbations.  Vulnerability to such ``adversarial attacks'', or ``adversarial examples'', has been conjectured to be due to the excessive linearity of deep networks. In this paper, we study this phenomenon in the setting of a linear classifier, and show that it is possible to exploit sparsity in natural data to combat $\ell_{\infty}$-bounded adversarial perturbations. Specifically, we demonstrate the efficacy of a sparsifying front end via an ensemble averaged analysis, and experimental results for the MNIST handwritten digit database. To the best of our knowledge, this is the first work to show that sparsity provides a theoretically rigorous framework for defense against adversarial attacks.
\end{abstract}

%% file: introduction.tex

Recent work in machine learning security points out the vulnerability of deep neural networks to adversarial perturbations \cite{szegedy2013intriguing,fawzi2017review,goodfellow2014adversarial,moosavi2016deepfool}. These perturbations can be designed to be barely noticeable to the human eye, but can cause large classification errors in state of the art deep networks.  While it is tempting to speculate that this vulnerability arises from the complex, nonlinear nature of deep networks, a more plausible explanation is that it is due to the excessive linearity of such networks \cite{goodfellow2014adversarial,moosavi2016deepfool,fawzi2017classification,ben2016expressivity}.  When we take a linear combination of the components of a high-dimensional input, small, adversarially chosen, perturbations of each component can add up to a large perturbation at the output.  Complex operations such as a rectified linear unit (ReLU) operating beyond its bias, or a sigmoid in its linear region, together with operations such as max pooling or average pooling, when cascaded through multiple stages, still amount to an approximately linear combination of the input. Of course, the coefficients of the linear combination exhibit some dependence on the input, but these can be viewed as on-off switches rather
than a change in the value of the coefficients: for example, whether the input is such that a ReLU unit is operating in its linear region, or 
the identity of the argument of the maximum in a max pooling unit.  This motivates us to take a step back in this paper, and study adversarial perturbations in the simplest
possible setting: a linear classifier.

Sparsity is an intuitively plausible mechanism: we understand that humans reject small perturbations by focusing on the key features that stand out. Our proposed approach is based on this intuition. In this paper we show via both theoretical results and experiments that a sparsity-based defense is effective against $\ell_{\infty}$-bounded perturbations. 

We consider a system consisting of a linear classifier and two participants: the adversary and the defender. The adversary perturbs the input data, with the goal of causing misclassification. The defender inserts a pre-processing function in order to attenuate the impact of the adversary. We propose a sparsifying front end as the preprocessing function and evaluate its performance in two scenarios: a ``semi-white box'' setting where the adversary designs the perturbation based on the linear model, but without accounting for the pre-processing, and a ``white box'' setting where the attack accounts for both the pre-processing function and the classifier.

{\bf Contributions}:
We develop a theoretical framework to assess and demonstrate the effectiveness of a sparsity-based defense against adversarial attacks. To the best of our knowledge, this is the first work to show that sparsity provides a rigorous foundation for defense against adversarial perturbations. Our main contributions in this paper are as follows:
\begin{itemize}
\item We quantify the achievable gain of the sparsity-based defense via an ensemble-averaged analysis based on a stochastic model for the linear classifier. As the main theoretical contribution of the paper, in Theorems \ref{theorem1} and \ref{theorem2} we show that with high probability, sparsity-based defense reduces the adversarial impact by a factor of $K/N$ in the semi-white box setting, and by $\bigo(K \polylog(N)/N)$ in the white box setting, where $K$ is the sparsity of the signal, and $N$ is the signal's dimension.
\item We demonstrate the robustness of our proposed defense through experimental results for binary classification using a linear SVM on the MNIST handwritten digit database. 
Small adversarial perturbations can render such a classifier useless (0\% accuracy), but our sparsity-based defense limits the damage to 1-4\% degradation in accuracy for the semi-white
and white box attacks, respectively.
\end{itemize}

%% file: related_work.tex

The existence of ``blind spots'' in deep neural networks \cite{szegedy2013intriguing} has been the subject of extensive recent study in machine learning literature \cite{fawzi2017review}. It was initially hypothesized that this phenomenon is due to the high complexity of neural networks, but work on linearization-based attacks \cite{goodfellow2014adversarial,moosavi2016deepfool} and decision boundaries of deep networks \cite{fawzi2017classification,ben2016expressivity} indicates that it is instead due to their excessive linearity. A variety of defenses have been proposed to combat adversarial attacks, including some that implicitly make use of sparsity-related techniques \cite{bhagoji2017dimensionality,das2017keeping}. The evaluations in such prior work have been purely empirical. Our analytical framework supplements these by providing a theoretical justification for systematic and explicit pursuit of sparsity-based defenses.
It is worth noting that sparsity has also been suggested purely as a means of improving classification performance \cite{makhzani2013ksparse}, which indicates that the performance
penalty for appropriately designed sparsity-based defenses could be minimal.

%% file: problem.tex

\subsection{Preliminaries}

We denote by $\bx \in \mathbb{R}^{N}$ a data sample with $K$-sparse representation in orthonormal basis $\Psi$ $\left(\,=\left[\bpsi_1,\bpsi_2,\dots,\bpsi_N\right]\,\right)$:
\begin{equation*}
\left\Vert \Psi^T\bx\right\Vert_0 \le K \qquad (K\ll N).
\end{equation*}
Given a linear model $\bw \in \mathbb{R}^{N}$, and denoting by $\bhatx$ a {\it modified} data sample, we define performance measure $\Delta$:
\begin{equation*}
\Delta \left( \bx, \bhatx\right) = |\bw^T \bhatx - \bw^T \bx|.
\end{equation*}

\subsection{System Model}

Now we describe our system (depicted in Fig. \ref{fig:main_block}) composed of two blocks, the adversary and the defense:

\vspace{-3mm}
\begin{figure}[htbp]
  	\centering
	\begin{tikzpicture}[
	rect/.style = {rectangle, draw=black!100, fill=cyan!18, thin, minimum height=7.5mm, minimum width=11mm},
	rect_smaller/.style = {rectangle, draw=black!100, fill=cyan!18, thin, minimum height=7.5mm, minimum width=8mm},
	circ/.style = {circle, draw=black!100, fill=cyan!18, thin, minimum size=7mm},
	outer/.style = {rounded corners=0.2cm, draw=black!100, dashed, inner sep = 2.2mm}
	]
	\node[] (x) {$\bx$};
	\node[circ] (plus) [right = 5mm of x] {$+$};
	\node[] (e) [above = 5mm of plus] {$\be$};
	\node[rect] (f) [right = 11mm of plus] {$f(\cdot)$};
	\node[rect] (w) [right = 11mm of f] {$\bw^T$};
	\node[] (out) [right = 5mm of w] {$\bw^T\hat{\bx}$};
	\node[outer, inner sep = 2.2mm, fit = (plus), label=below:\footnotesize{Adversary}] (attack) {};
	\node[outer,fit = (f) (w), label=below:\footnotesize{Defense}] (frontend) {};
	\draw[->] (x) -- (plus);
	\draw[->] (e) -- (plus); 
	\draw[->] (plus) -- node[anchor=south]{$\bbarx$} (f); 
	\draw[->] (f) -- node[anchor=south]{$\bhatx$}(w);
	\draw[->] (w) -- (out);
	\end{tikzpicture}
	\caption{Block diagram of the system.}
  	\label{fig:main_block}
\end{figure}
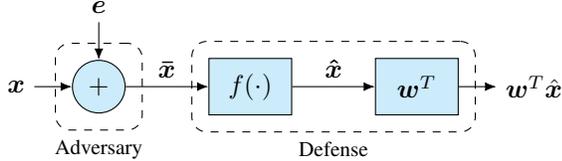
\vspace{-2mm}

\begin{itemize}
\item[$-$] The {\it adversary} induces an $\ell_{\infty}$-bounded additive perturbation $\be \in \mathbb{R}^N$ to data $\bx$, with the goal of maximizing $\Delta$:
\vspace{-2pt}
\begin{equation*}
\begin{split}
    \max_{\be}& \quad
    {\Delta \left( \bx, \bhatx \right)} \\
    \mathrm{s.t.}& \quad
    \left\Vert\be\right\Vert_\infty<\epsilon.
\end{split}
\end{equation*}
\item[$-$] The {\it defense} adds a pre-processing function $f:\mathbb{R}^N\to\mathbb{R}^N$ to the linear model $\bw$, with the goal of minimizing $\Delta$.
\end{itemize}

%% file: sparse_theory.tex

\subsection{Pre-processing Function}

Given a linear classifier, we propose a pre-processing function via a {\it sparsifying front end} to combat adversarial attacks. Figure \ref{fig:schematic} shows a block diagram of our model, composed of sparsity-based preprocessing and a linear machine learning model $\bw^T$. Function $\sparse(\cdot)$ enforces sparsity by retaining the $K$ coefficients largest in magnitude and zeroing out the rest. Since $\bx$ is $K$-sparse in orthonormal basis $\Psi$, we note that $\bhatx = \bx$ when there is no attack ($\be = \bm{0}$).

We define the following quantities:
\begin{align}
\support \left(\bx\right) \triangleq \supp \left(\sparse \left(\Psi^T\bx\right)\right), \nonumber \\
\proj \left(\be,\bx\right) \triangleq  \sum_{k \in \support(\bx)} \bpsi_k \bpsi_k^T\be, \nonumber
\end{align}
where  $\support \left(\bx\right)$ is the support of the $K$-sparse representation of $\bx$, and $\proj \left(\be,\bx\right)$ is the projection of $\be$ on the subspace spanned by $\support \left(\bx\right)$.

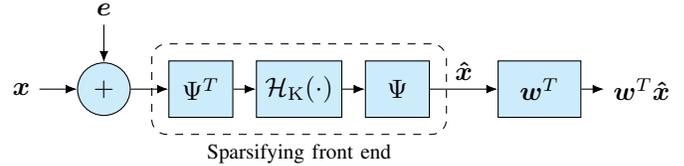
\begin{figure}[htbp]
	\vspace{-7pt}
  	\centering
	\begin{tikzpicture}[
	rect/.style = {rectangle, solid, draw=black!100, fill=cyan!18, thin, minimum height=7.5mm, minimum width=11mm},
	rect_smaller/.style = {rectangle, solid, draw=black!100, fill=cyan!18, thin, minimum height=7.5mm, minimum width=8.5mm},
	circ/.style = {circle, solid, draw=black!100, fill=cyan!18, thin, minimum size=7mm},
	outer/.style = {rounded corners=0.2cm, draw=black!100, dashed, inner sep = 2.2mm}
	]
	\node[] (x) {$\bx$};
	\node[circ] (plus) [right = 5mm of x] {$+$};
	\node[] (e) [above = 5mm of plus] {$\be$};
	\node[rect_smaller] (psit) [right = 5mm of plus] {$\Psi^T$};
	\node[rect_smaller] (hk) [right = 3mm of psit] {$\sparse(\cdot)$};
	\node[rect_smaller] (psi) [right = 3mm of hk] {$\Psi$};
	\node[rect] (w) [right = 9mm of psi] {$\bw^T$};
	\node[] (out) [right = 3mm of w] {$\bw^T\bhatx$};
	\node[outer,fit = (psit) (hk) (psi), label=below:\footnotesize{Sparsifying front end}] (frontend) {};
	\draw[->] (x) -- (plus);
	\draw[->] (e) -- (plus); 
	\draw[->] (plus) -- (psit); 
	\draw[->] (psit) -- (hk);
	\draw[->] (hk) -- (psi);
	\draw[->] (psi) -- node[anchor=south]{$\bhatx$} (w);
	\draw[->] (w) -- (out);
	\end{tikzpicture}
	\vspace{-15pt}
	\caption{Block diagram of sparsity-based defense}
	\label{fig:schematic}
\end{figure}

We also define the {\it high SNR regime} as the operating region where the additive perturbation does not shift the $K$-dimensional subspace of $\bx$:
\begin{align}  \label{eq:Complete_recovery}
\support(\bx)=\support(\bx+\be).
\end{align}
In Section \ref{sec:analytical}, Proposition \ref{prop:SNR_condition}, we characterize the conditions that guarantee (\ref{eq:Complete_recovery}).
Now assuming that we operate in the high SNR regime, we get
\begin{equation*}
\sparse \left(\Psi^T (\bx+\be)\right) = \sparse \left(\Psi^T \bx\right) + \bbare
=\Psi^T \bx + \bbare,
\end{equation*}
where
\begin{equation*}
     \bar{e}_k=
    \begin{cases}
      \bpsi_k^T \be , & \text{if}\ k \in \support(\bx) \\
      0, & \text{otherwise}.
    \end{cases}
  \end{equation*}
The output of the pre-processing function thus becomes
\begin{equation*}
\bhatx = \bx + \sum_{k \in \support(\bx)}\bpsi_k \bpsi_k^T \be 
\, = \, \bx +  \proj(\be,\bx).
\end{equation*}
Therefore, the performance measure or adversarial attack's impact will be
\begin{align} 
\Delta = \left| \bw^T\bhatx - \bw^T\bx \right| &= \left| \bw^T \proj(\be,\bx) \right| \nonumber \\
&= \left| \be^T \proj(\bw,\bx) \right|, \label{eq:Delta}
\end{align}
where (\ref{eq:Delta}) follows directly from the definition of $\proj(\be,\bx)$.

\subsection{Attacks and defenses}

We now compare the robustness of both the plain classifier and our proposed model against various attacks designed based on partial/full knowledge of the defense.

\begin{enumerate}[itemsep=5pt, topsep=5pt, label=\arabic*.,leftmargin=15pt]
\item \textbf{No front end:} Here the perturbed data is directly input to the ML classifier, i.e, $\Delta_0=\left| \bw^T\be \right|$. We use this scenario as a baseline to assess the efficacy of our defense. 

Assuming the adversary has knowledge of $\bw$, the most effective attack would be in the direction orthogonal to the classifier's decision boundary, subject to the $\ell_{\infty}$ constraint: 
\begin{equation*}
\be=\epsilon \sign(\bw).
\end{equation*}
This yields
\begin{equation*}
\Delta_0=\epsilon \, \left\Vert\bw\right\Vert_1.
\end{equation*}

\item\textbf{Semi-white box attack:} In this scenario the defender employs the sparsifying front end, but the adversary designs the perturbation based on knowledge of $\bw$ alone. Hence the perturbation remains
\begin{equation*}
\be_{\mathrm{SW}}=\epsilon \sign \left(\bw\right).
\end{equation*}
Using (\ref{eq:Delta}), we get the impact of the attack as follows:
\begin{equation*}
\Delta_{\mathrm{SW}} = \epsilon \left| \sign(\bw^T) \proj(\bw,\bx) \right|.
\end{equation*}

\item\textbf{White box attack:} Here the adversary has knowledge of both $\bw$ and the front end, and designs perturbations accordingly. This results in the following optimization problem:
\vspace{-5pt}
\begin{equation*}
\begin{split}
    \max_{\be}& \quad
    {\left| \be^T \proj(\bw,\bx) \right|} \\
    \mathrm{s.t.}& \quad
    \left\Vert\be\right\Vert_\infty<\epsilon.
\end{split}
\end{equation*}
The optimal perturbation is
\begin{equation*}
\be_{\mathrm{W}}=\epsilon \sign \left(\proj \left(\bw,\bx \right)\right),
\end{equation*}
and its impact becomes
\begin{equation*}
\Delta_{\mathrm{W}}=\epsilon \, \left\Vert\proj (\bw,\bx)\right\Vert_1.
\end{equation*}
Thus, instead of aligning with $\bw$, $\be_{\mathrm{W}}$ is aligned to the projection of $\bw$ on the subspace that $\bx$ lies in. 
\end{enumerate} 

%% file: analytical_results.tex

\subsection{Characterizing the High SNR Regime}

\begin{proposition} \label{prop:SNR_condition}
For sparsity level K, the sparsifying front end preserves the input coefficients if the following SNR condition holds:
\begin{equation*}
\mathrm{SNR} \, \triangleq \, \frac{\lambda}{\epsilon} \,> 2M,
\end{equation*}
where $\lambda$ is the magnitude of the smallest non-zero entry of $\sparse(\Psi^T x)$ and $M = \max_l{\left\Vert \bpsi_l\right\Vert_1}$.
\end{proposition}

\begin{proof}
It is easy to see that (\ref{eq:Complete_recovery}) is equivalent to
\begin{equation*}
\min_{i \in \support(\bx)} \left|\bpsi^T_i (\bx+\be) \right| > \max_{j \notin \support(\bx)} \left|\bpsi_j^T(\bx+\be)\right| 
 = \max_{j \notin \support(\bx)} \left|\bpsi_j^T\be\right|, 
\end{equation*} 
where the equality follows from the definition of $\support(\cdot)$.
Denoting the optimal indices by $i_{\subzero}$ and $j_{\subzero}$, we use triangle inequality to obtain
$| \bpsi^T_{i_{_\subzero}}\bx | > | \bpsi_{i_{_\subzero}}^T\be | + | \bpsi_{j_{_\subzero}}^T\be |$.
The proposition follows by applying H\"{o}lder's inequality and using the $\ell_{\infty}$-bound on $\be$.
\end{proof} 

\begin{remarks}\
\begin{enumerate}[label=\arabic*.,leftmargin=15pt]
\item The SNR condition is easier to satisfy for bases with sparser, or more localized, basis functions (smaller $M$). For example, we expect a wavelet basis to be better than a DCT basis.
\item \label{rem:snr} When $\bx$ is approximately $K$-sparse, choosing smaller $K$ allows the SNR condition to hold for larger perturbations, but at the expense of higher signal perturbation. These must be traded off to optimize classification performance.
\end{enumerate}
\end{remarks}

All of our subsequent analysis in this section is based on the assumption that the SNR condition in Proposition \ref{prop:SNR_condition} holds. In this case, the sparsifying front end is signal-preserving, hence the output distortion can be quantified solely by analyzing its effect on the adversarial perturbation. In our experiments with MNIST data, we find that the SNR condition is approximately satisfied for the range of $K$ that works most effectively (1-5\% of the coefficients in a wavelet basis).

\subsection{Ensemble Averaged Performance}

We now provide an analysis that quantifies the robustness provided by sparsification over an ensemble of linear classifiers, by imposing
a stochastic model for $\bw$.

\begin{assumption}
For $\bw = \left(w_1,\dots, w_N\right)^T$, we model the $\{ w_i ,\, i=1,\dots,N \}$ as i.i.d., with zero mean and median: $\E\left[w_1 \right]=0$ and $\E\left[\sign \left(w_1\right) \right] = 0$.
Let $\E\left[|w_1|\right] = \mu$ and $\E\left[w_1^2\right] = \sigma^2$.
\end{assumption}

\subsubsection{Semi-White Box Attack}

\begin{theorem} \label{theorem1}
As $K$ approaches infinity, $\Delta_{\mathrm{SW}}/K$ converges to $\mu$ in probability,
i.e.\
\begin{equation*}
\lim_{K\to\infty} \prob \left( \left| \frac{\Delta_{\mathrm{SW}}}{K} - \mu \right| \leq \delta \right) = 1 \quad \forall \; \delta >0.
\end{equation*}
\end{theorem}

\begin{remark}\label{rem:scaling_semiwhite}
After sparsification, the impact $\Delta_{\mathrm{SW}}$ of the adversarial perturbation scales linearly with the sparsity level $K$. Thus, the sparsifying front end provides an attenuation of $K/N$ on the effect of the semi-white box adversarial attack.
\end{remark}

\begin{proof}
Assuming without loss of generality that $\support(\bx) = \{1,2,\dots,K\}$, the output distortion can be written as $\,\Delta_{\mathrm{SW}} = |Z_K|, \; Z_K = \sum_{i=1}^K U_i V_i$, where
\begin{equation*}
U_i = \sum_{m=1}^N \psi_i \left[m\right] \, w_m,\;
V_i = \sum_{m=1}^N \psi_i \left[m\right] \, \sign(w_m),\;\, i=1,\dots,K.
\end{equation*}
We now state the following lemma:

\begin{lemma} \label{var_covar}
The mean and variance of $Z_K$ are bounded by linear functions of $K$:
\begin{equation*}
\E(Z_K) = K \mu, \quad \var(Z_K) \leq K \left(\sigma^2 + \mu^2\right). 
\end{equation*}
\end{lemma}

\begin{proof}
We observe that for $i,j \in \{1,\dots,K\}$, $\E\left[U_i V_i\right] = \mu$,
\begin{align*}
&\var\left(U_i V_i\right) = \sigma^2 + \mu^2 -2 \mu^2 \sum_{m=1}^N \psi_i^4 \left[m\right],  \\
&\cov\left(U_i V_i,U_j V_j\right) = -2\mu^2 \sum_{m=1}^N \psi_i^2\left[m\right] \, \psi_j^2\left[m\right], \quad i \neq j.
\end{align*}
Hence we get $\E[Z_K] = K \mu$, and
\begin{align*}
\var(Z_K) &= \sum_{i=1}^K \var\left(U_i V_i\right) - \sum_{\substack{i,j=1 \\ i \neq j}}^K \cov\left(U_i V_i,U_j V_j\right) \\
&= K \left(\sigma^2 + \mu^2\right) -2\mu^2 \sum_{m=1}^N \, \sum_{i,j=1}^K \psi_i^2\left[m\right] \, \psi_j^2\left[m\right] \\
&\leq K \left(\sigma^2 + \mu^2\right).
\end{align*}
\end{proof}

\noindent We now apply Chebyshev's inequality to $Y_K = {Z_K / K}$, noting that $\E\left[Y_K\right]=\mu$ and $\var\left(Y_K\right) \leq {\left(\sigma^2 + \mu^2\right) / K}$:
\begin{align*}
\prob \left(|Y_K-\mu| \leq \delta \right) \geq 1 - \frac{1}{K} \left(\frac{\sigma^2+\mu^2} {\delta^2}\right) \quad \forall \; \delta \geq 0.
\end{align*}
The theorem follows by applying the sandwich theorem to the above inequality as ${K\to\infty}$, observing that $|\Delta_{\mathrm{SW}}/K - \mu| = ||Y_K| - \mu| \leq |Y_K - \mu|$.
\end{proof}

\subsubsection{White Box Attack}

\begin{lemma} \label{lemma:triangle}
An upper bound on the white box attack distortion is given by
\begin{equation*}
\Delta_{\mathrm{W}} \leq \sum_{k=1}^K \left| \bpsi_k^T \bw \right| \left\Vert \bpsi_k \right\Vert_1.
\end{equation*}
\end{lemma}

\begin{proof}
\begin{align*}
\Delta_{\mathrm{W}} &= \sum_{i=1}^N \, \left| \sum_{k=1}^K \left( \bpsi_k^T \bw \right) \psi_k\left[i\right] \right|
\\ & \leq \sum_{i=1}^N \sum_{k=1}^K \left| \bpsi_k^T \bw \right| \left| \psi_k\left[i\right] \right|
= \sum_{k=1}^K  \left| \bpsi_k^T \bw \right|  \left\Vert \bpsi_k \right\Vert_1.
\end{align*}
\end{proof}

\begin{remarks}\label{rem:upper_bound_white}\
\begin{enumerate}[label=\arabic*.,leftmargin=15pt]
\item The upper bound is exact if the supports of the $K$ selected basis functions do not overlap. In our MNIST experiments, this is approximately satisfied for the range of $K$ that works most effectively (1-5\% of the coefficients in a wavelet basis).
\item Since the upper bound has $K$ terms, the distortion cannot grow slower than $K$. As stated in the following theorem, however, if the basis functions are ``localized'' with $\ell_1$ norms that do not scale too fast with $N$, then the output distortion scales as $\bigo\left(K \polylog(N)\right)$. 
\end{enumerate}
\end{remarks}

\begin{theorem} \label{theorem2}
With high probability,
\begin{equation*}
\Delta_\mathrm{W} \leq \bigo\left(K \polylog(N)\right),
\end{equation*}
under the assumptions $\left\Vert \bpsi_k \right\Vert _1 = \bigo(\log N)$, $\left\Vert \bpsi_k \right\Vert _\infty = \smallo(1)$ $\,\forall\, k\in\{1,2,\dots,K\}$, and $\left\Vert \bw \right\Vert _\infty = \bigo(1)$. Equivalently,
\begin{equation*}
\lim_{N \to \infty} \Pr\left(\Delta_\mathrm{W} \leq \bigo\left(K \polylog(N)\right)\right)=1.
\end{equation*}
\end{theorem}

\begin{proof} \label{white_proof}
Letting $Z_k = \bpsi_k^T \bw$, we first state the following lemma:

\begin{lemma} \label{lemma:CLT}
$Z_k \rightarrow \mathcal{N}(0,\,\sigma^2)$ in distribution.
\end{lemma}

\begin{proof} \label{CLT_proof}
We show that we can apply Lindeberg's version of the central limit theorem, noting that $Z_k = \sum_{i=1}^N Y_i$, where $Y_i = \psi_k \left[i \right] w_i$ are independent random variables with $\E[Y_i]=0$ and $\var(Y_i)=\sigma_i^2$, with $\sum_{i=1}^N \sigma_i^2=\sigma^2$.

Now, given $\delta>0$, we investigate the following quantity in order to check Lindeberg's condition:
\begin{equation*}
L(\delta,N) = \frac{1}{\sigma^2} \sum_{i=1}^N \E \left[ Y_i^2  \mathbbm{1}_{\left\{|Y_i|> \delta \sigma \right\}} \right].
\end{equation*}
From the $\ell_{\infty}$ assumptions on $\bpsi_k$ and $\bw$, we observe that
\begin{align*}
 \E \left[ \psi_k^2\left[i \right] w_i^2 \mathbbm{1}_{\left\{|Y_i|> \delta \sigma \right\}}\right] &\leq \smallo^2(1)\bigo^2(1) \Pr{\left(|Y_i|> \delta \sigma \right)} \\
  &= \smallo^2(1)\bigo^2(1) \Pr{\left(|w_i|> \frac{\delta \sigma}{\smallo(1)} \right)}.
 \end{align*}
Also note that $\forall\, \delta>0, \; \exists\, M$ s.t. $\forall N>M$, $|w_i |< \delta \sigma/\smallo(1)$ $\forall \,i \in \{1,\dots,N\}$. Hence we get $\lim_{N \to \infty} L(\delta,N) =0$, which is Lindeberg's condition. 
\end{proof}

\noindent From Lemmas \ref{lemma:triangle} and \ref{lemma:CLT}, we get
\begin{align*}
&\Pr\left(\Delta_\mathrm{W} > \delta\right) \leq \Pr \left( \sum_{k=1}^K \left|Z_k\right| \left\Vert\bpsi_k\right\Vert_1 > \delta \right) \\
&\leq \Pr \left( \bigcup_{k=1}^K \left\{\left|Z_k\right| > \frac{\delta}{K \left\Vert\bpsi_k\right\Vert_1} \right\} \right) 
\leq \sum_{k=1}^K \Pr \left( \left|Z_k\right| > \frac{\delta}{K \left\Vert\bpsi_k\right\Vert_1} \right) \\
&= \sum_{k=1}^K 2 Q\left(\frac{\delta}{\sigma K \left\Vert\bpsi_k\right\Vert_1}\right)
= 2K Q\left(\frac{\delta}{\sigma} \bigo\left(\frac{1}{K \log N}\right)\right),
\end{align*}
where $Q(x) = \frac{1}{\sqrt{2 \pi}} \int_{x}^{\infty}  e^{-t^2/2}dt$, and we have used the $\ell_1$ assumption on $\bpsi_k$ in the last step. The theorem follows by setting $\delta = \bigo(K \polylog(N))$ and applying the sandwich theorem as $N\to \infty$.
\end{proof}

%% file: experiments.tex

In this section we demonstrate the efficacy of sparsifying front ends on an inference task where our analysis directly applies: classification of digit pairs from the MNIST handwritten digit database \cite{lecun1998gradient} via linear SVM.\footnote{Code is available at {\url{https://github.com/soorya19/sparsity-based-defenses/}}.}

\subsection{Setup}
 
We consider the task of discriminating between digits $d_1$ and $d_2$, where $d_1\neq d_2 \in \{0, 1, \dots 9\}$. The dataset of interest is $\mathcal{X} = \{\bx: \mathcal{L}(\bx) \in \{d_1,d_2\}\}$, where $\bx$ denotes the images normalized to ${[-1,1]}$ and $\mathcal{L}(\bx)$ the true labels. We divide $\mathcal{X}$ into training and test sets $\mathcal{X}_{tr}, \mathcal{X}_{te}$ in a 3:1 ratio.

We train a linear SVM classifier $f(\cdot)$ on $\mathcal{X}_{tr}$ and obtain class predictions $\mathcal{C}(\cdot)$ as follows:
\begin{equation*}
f(\bx) = \bw^T \bx + b, \qquad
     \mathcal{C}(\bx)=
    \begin{cases}
      \; d_1, & f(\bx) < 0 \\
      \; d_2, & f(\bx) > 0.
    \end{cases}
\end{equation*}
In the scenario without front end, we consider the adversarial perturbation $\be = \epsilon \sign \left(\bw\right)$ on $\mathcal{X}_{te}$, where the ``direction'' of

 the attack is opposite that of the correct class:

\begin{equation*}
     \bbarx\,=
    \begin{cases}
      \; \bx + \be, & \mathcal{L}(\bx) = d_1 \\
      \; \bx - \be, & \mathcal{L}(\bx) = d_2
    \end{cases} \qquad \forall \; \bx \in \mathcal{X}_{te}.
\end{equation*}
In practice, the adversary usually only has access to $\mathcal{C}(\cdot)$ and not $\mathcal{L}(\cdot)$ for the test set. Hence this provides an upper bound for the classification error.

For the sparsifying front end, we use the Cohen–Daubechies–Feauveau 9/7 wavelet \cite{cohen1992biorthogonal} and impose sparsity in the wavelet domain. We retrain the SVM with the sparsified $\mathcal{X}_{tr}$ for various values of $\rho = K/N$, and evaluate the impact of semi-white box and white box attacks on $\mathcal{X}_{te}$.

\subsection{Results}

We begin with $3$ vs. $7$ classification. Without the front end, an attack with $\epsilon = 0.25$ completely overwhelms the classifier, reducing accuracy from 98.20\% to 0\%. 
Fig. \ref{fig:image_adv} shows a sample image before and after attack.

\begin{figure}[t]
  \centering
  \includegraphics[width=0.5\textwidth]{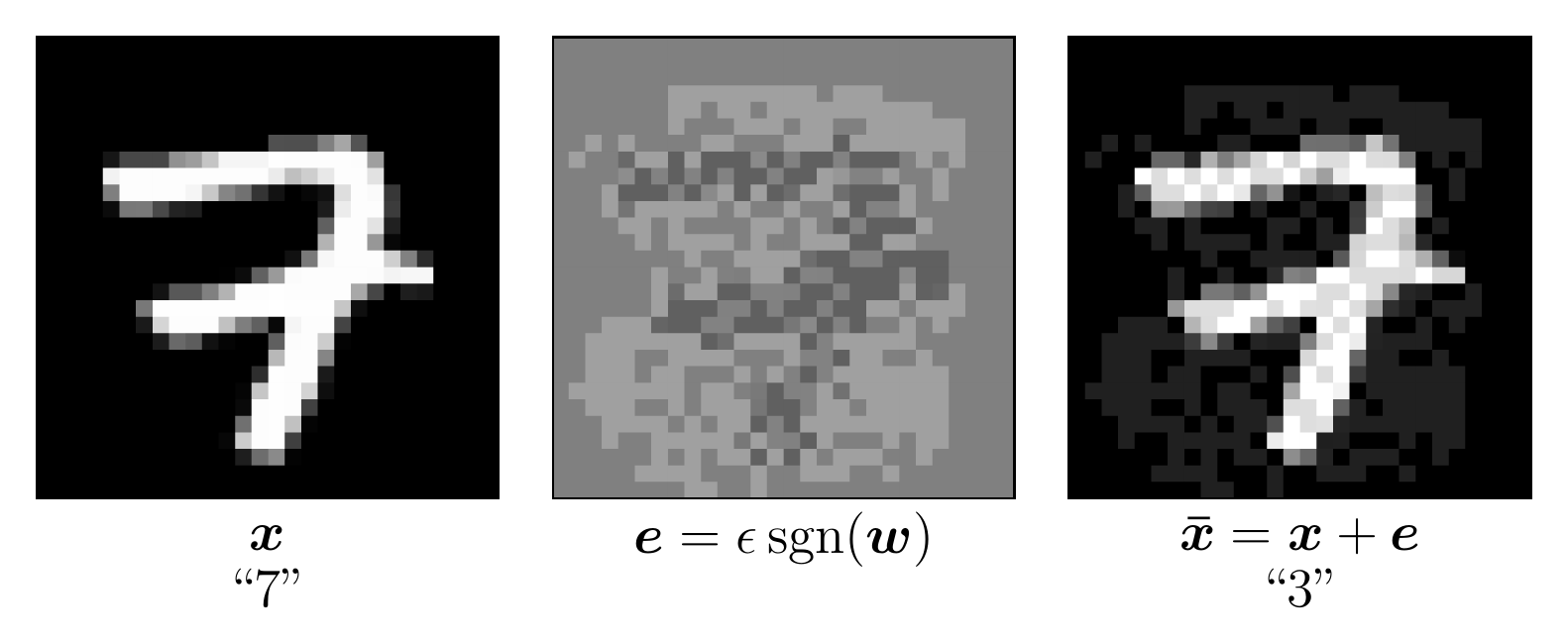}
  \vspace{-25pt}
  \caption{Sample image before and after attack ($\epsilon = 0.25$).
   The adversarial perturbation causes digit 7 to be misclassified as 3.}
  \label{fig:image_adv}
\end{figure}

Insertion of the sparsifying front end confers resiliency to attacks: at low values of $\rho$, accuracy is restored to near-baseline levels. 
The optimal value of $\rho$ must trade off signal distortion versus perturbation attenuation.  We find $\rho = 2$\% to be the best choice for the $3$ versus $7$ scenario,
and report on the accuracies obtained in Table \ref{table:binary}. Results for other digit pairs show a similar trend. Insertion of the front end greatly improves resilience to adversarial attacks. The optimal value of $\rho$ lies between $1-5$\%, with $\rho=2$\% working well for all scenarios. 

To give a concrete feel of the front end at work, Fig. \ref{fig:image_adv_sp} shows an example image, the attacked image, and the attacked image after
sparsification.

Fig. \ref{fig:acc_vs_sp} reports on accuracy as a function of $\rho$.  At the low values of $\rho$ that we are interested in, the white box attack is more damaging than
the semi-white box attack.  At higher $\rho$, a white box attack performs worse than the semi-white box attack: the high SNR condition in Proposition \ref{prop:SNR_condition} is no longer satisfied, hence the white box attack is attacking the ``wrong subspace.'' It is easy to devise iterative white box attacks that do better, but we do not discuss them here because the scenario of large $\rho$ is not of practical interest, since it does not provide enough attenuation of the adversarial perturbation.

\begin{figure}[b]
  \vspace{-10pt}
  \centering
  \includegraphics[width=0.5\textwidth]{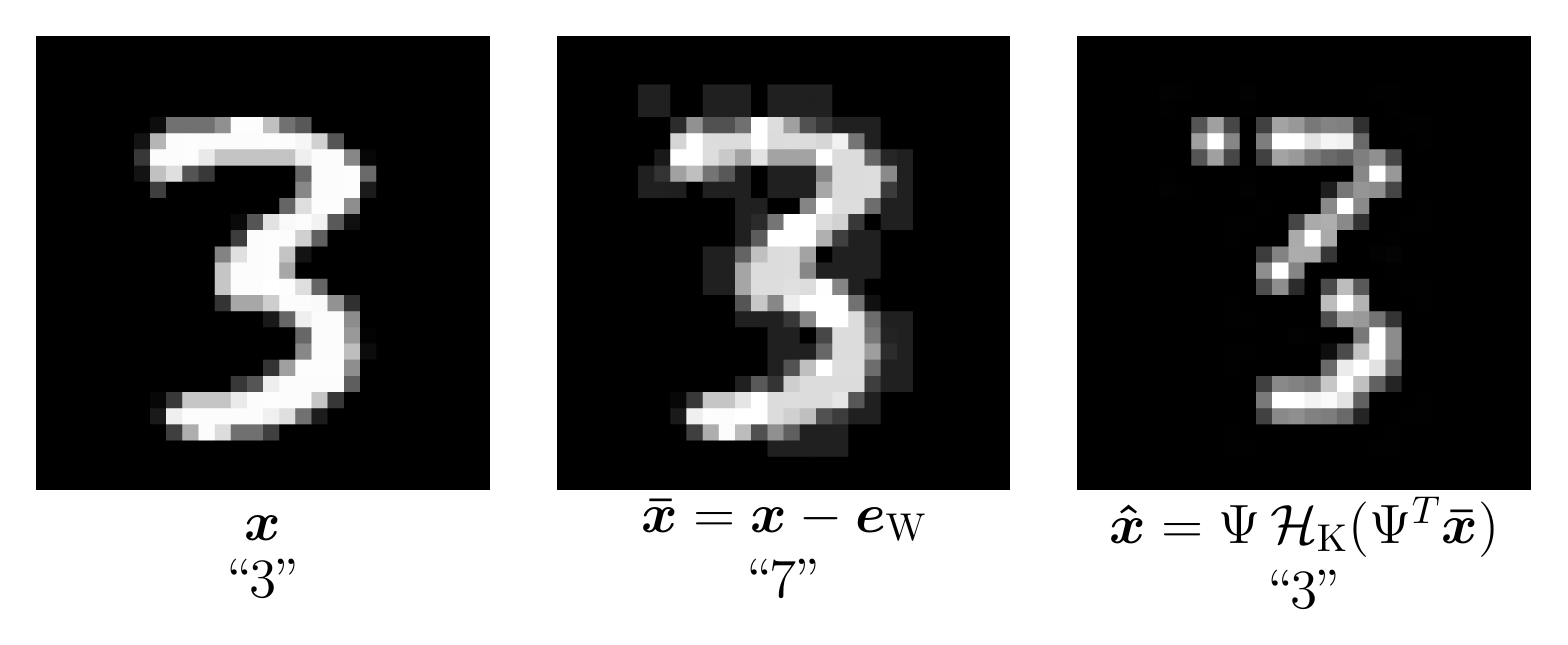}
  \vspace{-25pt}
  \caption{Sample image after white box attack ($\epsilon = 0.25$) and front end (2\%-sparse). The sparsified version of the attacked image resists misclassification.}
  \vspace{-5pt}
  \label{fig:image_adv_sp}
\end{figure}

\begin{table}[tb]
\centering
\caption{Binary classification accuracies ($3$ vs. $7$)}
\vspace{-5pt}
\label{table:binary}
\begin{tabular}{@{}lcc@{}}
\toprule
 & No front end & \begin{tabular}[c]{@{}c@{}}Sparsifying front end \\ ($\rho = \,$2\%)\end{tabular} \\ \midrule
No attack & 98.20\% & 98.59\% \\
Semi-white box attack & 0\% & 97.31\% \\
White box attack & 0\% & 94.62\% \\ \bottomrule                          
\end{tabular}
\end{table}

\begin{figure}[tb]
  \centering
  \includegraphics[width=0.5\textwidth]{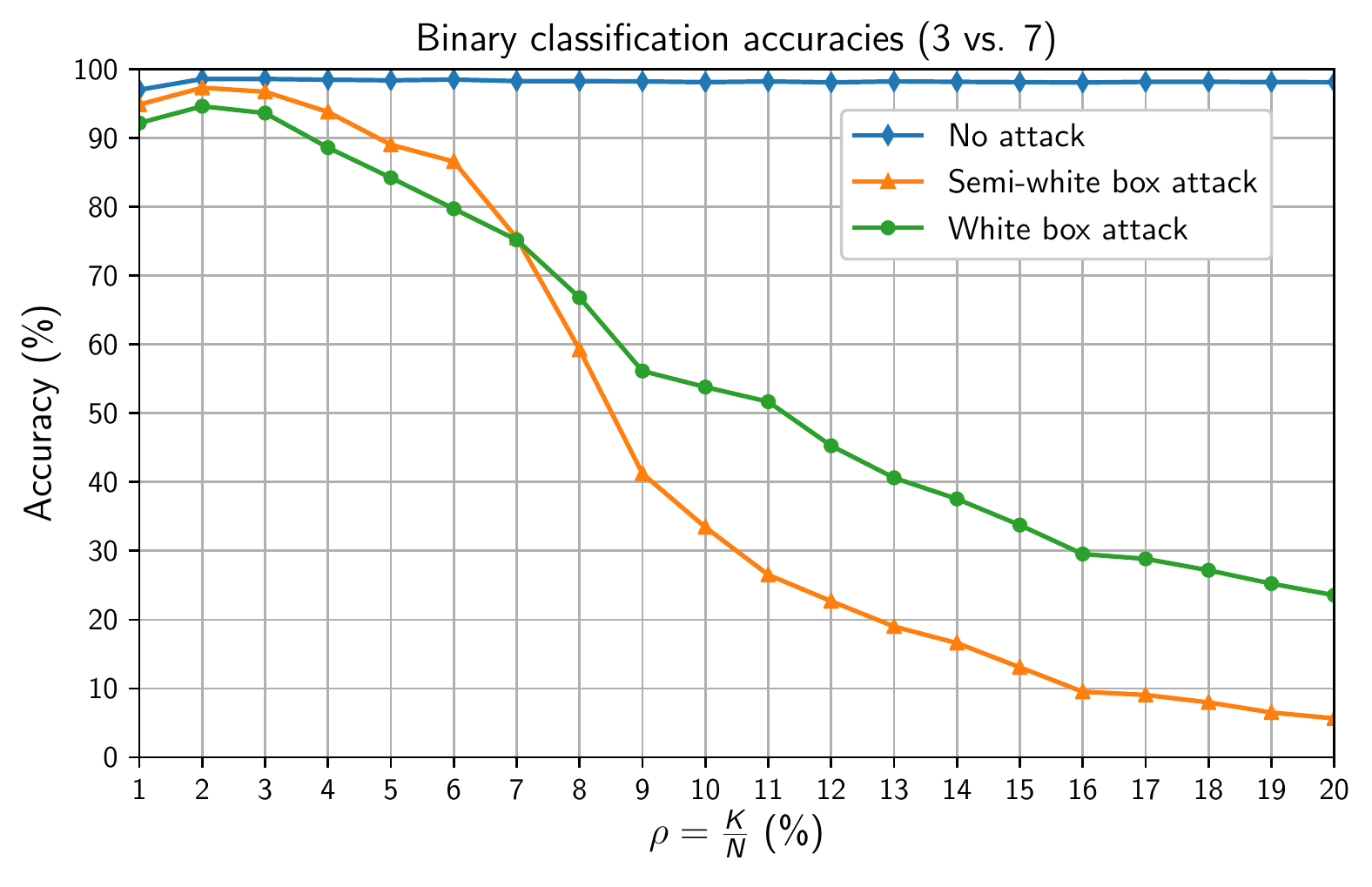}
  \vspace{-20pt}
  \caption{Binary classification accuracies as a function of front end sparsity. All attacks use $\epsilon = 0.25$. Effectiveness of the front end decreases with increase in $\rho$.}
  \vspace{-10pt}
  \label{fig:acc_vs_sp}
\end{figure}

%% file: acknowledgement.tex

This work was supported in part by the National Science Foundation under grants CNS-1518812 and CCF-1755808, by Systems on Nanoscale Information fabriCs (SONIC), one of the six SRC STARnet Centers, sponsored by MARCO and DARPA, and by the UC Office of the President under grant No. LFR-18-548175.